\pdfoutput=1
\documentclass{article}
\PassOptionsToPackage{round}{natbib}
\usepackage[preprint]{nips_2018}

\usepackage{color}
\definecolor{darkblue}{rgb}{0,0.08,0.45}
\usepackage[colorlinks=true,allcolors=darkblue]{hyperref}       % hyperlinks
\usepackage{url}            % simple URL typesetting
\usepackage{booktabs}       % professional-quality tables
\usepackage{amsfonts}       % blackboard math symbols
\usepackage{nicefrac}       % compact symbols for 1/2, etc.
\usepackage{microtype}      % microtypography

\renewcommand{\cite}{\citep}
\usepackage{amsmath}
\usepackage{amsthm}
\usepackage{newtxtext,newtxmath}
\usepackage{braket}
\usepackage{mathtools}
\usepackage[ruled,vlined,linesnumbered]{algorithm2e}
\usepackage{wrapfig}
\usepackage{siunitx}
\usepackage{tabularx}
\usepackage{etoolbox}% <-- for bold fonts
\newcommand{\ubold}{\fontseries{b}\selectfont}% <-- for bold fonts
\robustify\ubold% <-- for bold fonts

\usepackage{subcaption}

\newtheoremstyle{mydefinition}
  {5.5pt} % Space above
  {0pt} % Space below
  {} % Body font
  {} % Indent amount
  {\bfseries} % Theorem head font
  {.} % Punctuation after theorem head
  {.5em} % Space after theorem head
  {} % Theorem head spec (can be left empty, meaning 'normal')
\newtheoremstyle{mytheorem}
  {5.5pt} % Space above
  {0pt} % Space below
  {\itshape} % Body font
  {} % Indent amount
  {\bfseries} % Theorem head font
  {.} % Punctuation after theorem head
  {.5em} % Space after theorem head
  {} % Theorem head spec (can be left empty, meaning 'normal')
\theoremstyle{mydefinition}

\theoremstyle{mytheorem}
\newtheorem{theorem}{Theorem}

\newcounter{enum2}
\renewenvironment{enumerate}{%
\begin{list}%
 {\arabic{enum2}.\ \,}{%
 \usecounter{enum2}
 \setlength{\itemindent}{0pt}
 \setlength{\leftmargin}{12pt}
 \setlength{\rightmargin}{0pt}
 \setlength{\labelsep}{0pt}
 \setlength{\labelwidth}{20pt}
 \setlength{\itemsep}{0pt}
 \setlength{\parsep}{0pt}
 \setlength{\listparindent}{0pt}
 \setlength{\topsep}{0pt}
 }}{\end{list}}

\newcommand{\R}{\mathbb{R}}

\newcommand{\Cc}{\mathcal{C}}

\newcommand{\Tc}{\mathcal{T}}

\newcommand{\Pc}{\mathcal{P}}

\newcommand{\KL}{D_{\mathrm{KL}}}
\newcommand{\rank}{\pi}
\renewcommand{\phi}{\varphi}
\renewcommand{\vec}[1]{\boldsymbol{#1}}

\newcommand{\argmin}{\mathop{\mathrm{argmin}}}
\newcommand{\ip}{\perp\!\!\!\perp}
\newcommand{\Bon}{\mathrm{Bon}}
\newcommand{\Tar}{\mathrm{Tar}}

\newcommand{\med}{\mathop{\mathrm{med}}}

\usepackage{color}

\title{Finding Statistically Significant Interactions\\ between Continuous Features}

\author{
  Mahito Sugiyama \\
  National Institute of Informatics\\
  JST PRESTO\\
  \texttt{mahito@nii.ac.jp} \\
  \And
  Karsten Borgwardt\\
  D-BSSE, ETH Z{\"u}rich\\
  SIB Swiss Institute of Bioinformatics\\
  \texttt{karsten.borgwardt@bsse.ethz.ch}
}

\begin{document}
\maketitle

\begin{abstract}
The search for higher-order feature interactions that are statistically significantly associated with a class variable is of high relevance in fields such as Genetics or Healthcare, but the combinatorial explosion of the candidate space makes this problem extremely challenging in terms of computational efficiency and proper correction for multiple testing. While recent progress has been made regarding this challenge for binary features, we here present the {\it first solution for continuous features}. We propose an algorithm which overcomes the combinatorial explosion of the search space of higher-order interactions by deriving a lower bound on the $p$-value for each interaction, which enables us to massively prune interactions that can never reach significance and to thereby gain more statistical power.
In our experiments, our approach efficiently detects all significant interactions in a variety of synthetic and real-world datasets.
\end{abstract}

\section{Introduction}
A big challenge in high-dimensional data analysis is the search for features that are statistically significantly associated with the class variable, while accounting for the inherent multiple testing problem. This problem is relevant in a broad range of applications including natural language processing, statistical genetics, and healthcare.
To date, this problem of \emph{feature selection}~\cite{Guyon03} has been extensively studied in statistics and machine learning, including the recent advances in selective inference~\cite{Taylor15}, a technique that can assess the statistical significance of features selected by linear models such as the Lasso~\cite{Lee2016}.

However, current approaches have a crucial limitation: They can find only \emph{single features or linear interactions of features}, but it is still an open problem to find \emph{patterns}, that is, \emph{multiplicative} (\emph{potentially higher-order}) \emph{interactions between features}.
A relevant line of research towards this goal is \emph{significant} (\emph{discriminative}) \emph{pattern mining}~\cite{Terada13,Llinares2015KDD,Papaxanthos16,Pellegrina18}, which tries to find statistically associated feature interactions while controlling the \emph{family-wise error rate} (\emph{FWER}), the probability to detect one or more false positive patterns.
However, all existing methods for significant pattern mining only apply to combinations of \emph{binary or discrete} features, and none of methods can handle real-valued data, although such data is common in many applications.
If we binarize data beforehand to use existing significant pattern mining approaches, a binarization-based method may not be able to distinguish (un)correlated features (see Figure~\ref{figure:example}).

To date, there is no method that can find all higher-order interactions of \textit{continuous} features that are significantly associated with an output variable and that accounts for the inherent multiple testing problem.
To solve this problem,
one has to address the following three challenges:
\begin{enumerate}
 \item\label{prob:pval} How to assess the significance for a multiplicative interaction of continuous features?
 \item\label{prob:mcp} How to perform multiple testing correction? In particular, how to control the FWER (family-wise error rate), the probability to detect one or more false positives?
 \item\label{prob:enum} How to manage the combinatorial explosion of the candidate space, where the number of possible interactions is $2^d$ for $d$ features?
\end{enumerate}
The second problem and the third problem are related with each other: If one can reduce the number of combinations by pruning unnecessary ones, one can gain statistical power and reduce false negative combinations while at the same time controlling the FWER.
Although there is an extensive body of work in statistics on multiple testing correction including the FDR~\cite{Hochberg88,Benjamini95} and also several studies on approaches to interaction detection~\cite{Bogdan15,Su16}, none of these approaches addresses the problem of dealing with the combinatorial explosion of the $2^d$-dimensional search space when trying to finding higher-order significant multiplicative interactions.

\emph{Our goal in this paper is to present the first method, called C-Tarone, that can find all higher-order interactions between continuous features that are statistically significantly associated with the class variable, while controlling the FWER.}

Our approach is to use the \emph{rank order statistics} to directly estimate the probability of joint occurrence of each feature combination from continuous data, which is known as \emph{copula support}~\cite{Tatti13}, and apply a \emph{likelihood ratio test} to assess the significance of association between feature interactions and the class label, which solves the problem~\ref{prob:pval}.
We present the tight lower bound on $p$-values of association, which enables us to prune unnecessary interactions that can never be significant through the notion of \emph{testability} proposed by~\citet{Tarone90}, and can solve both problems~\ref{prob:mcp} and~\ref{prob:enum}.

This paper is organized as follows:
We introduce our method C-Tarone in Section~\ref{sec:method}. We introduce a likelihood ratio test as a statistical association test for interactions in Section~\ref{subsec:test}, analyze multiple testing correction using the testability in Section~\ref{subsec:mcp}, and present an algorithm in Section~\ref{subsec:alg}.
We experimentally validate our method in Section~\ref{sec:exp} and summarize our findings in Section~\ref{sec:conclusion}.

\begin{figure}[t]
 \centering
 \includegraphics[width=.8\linewidth]{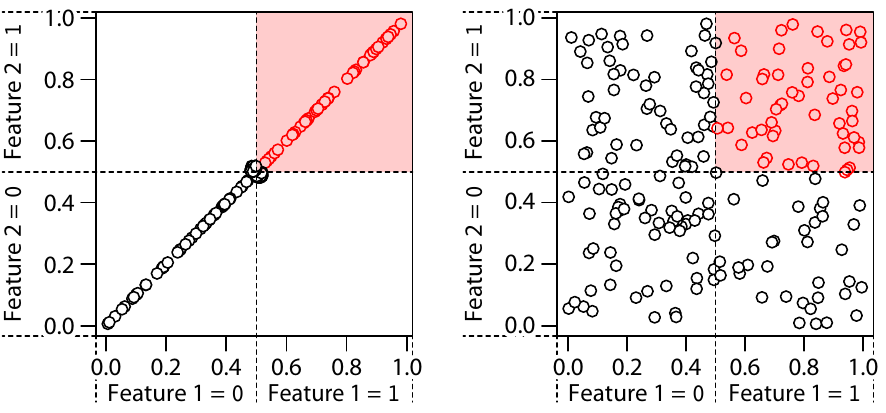}
 \caption{Although there is a clear correlation between Features 1 and 2 in the left panel and none in the right panel, after median-based binarization of Feature 1 and Feature 2, the estimated probability of the occurrence of the feature combination (the number of points for which Feature $1 =$ Feature $2 = 1$, or the number of points in the red box) will be exactly the same in both examples.
 Hence if the left case is a significant interaction, the right uncorrelated case also becomes a significant interaction in binarization-based methods.}
 \label{figure:example}
\end{figure}

\section{The Proposed Method: C-Tarone}\label{sec:method}
Given a supervised dataset $D = \{\,(\vec{v}_1, y_1), (\vec{v}_2, y_2)$, $\dots, (\vec{v}_N, y_N)\,\}$, where each data point is a pair of an $d$-dimensional vector $\vec{v}_i = (v_{i}^1, v_{i}^2, \dots, v_{i}^d) \in \R^d$ and a binary class label $y_i \in \{0, 1\}$.
We denote the set of features by $V = \{1, 2, \dots, d\}$ and the power set of features by $2^V$.
For each feature $j \in \{1, 2, \dots, d\}$, we write $\vec{v}^j = (v_1^j, v_2^j, \dots, v_N^j)$, which is the $N$-dimensional vector composed of the $j$th feature of the dataset $D$.

Our goal is to find \emph{every multiplicative feature interaction that is associated with class labels}.
To tackle the problem, first we measure the joint occurrence probability $\eta(J) \in \R$ of a feature combination $J \in 2^V$ in a dataset.
For each combination $J \in 2^V$, the size $|J|$ corresponds to the order of an interaction, and we find arbitrary-order interactions in the form of combinations.
If data is not real-valued but binary, that is, $v_i^j \in \{0, 1\}$, this problem is easily solved by measuring the \emph{support} used in frequent itemset mining~\cite{Agrawal94,Aggarwal14FPM} .
Each feature combination $J \in 2^V$ is called an \emph{itemset}, and the support of $J$ is defined as $\eta(J) = (1 / N) \sum_{i = 1}^N \prod_{j \in J} v_i^j$, which corresponds to joint probability of $J$.
The support is always from $0$ to $1$, and we can introduce a binary random variable $X_J$ corresponding to the joint occurrence of $J$, where $X_J = 1$ if features in $J$ jointly occurs and $X_J = 0$ otherwise, $\eta(J)$ corresponds to the empirical estimate of the probability $\Pr(X_J = 1)$.

This approach can be generalized to continuous (real-valued) data using the \emph{copula support}, which is the prominent result given by~\citet{Tatti13} in the context of frequent pattern mining from continuous data.
The copula support allows us to define the \emph{binary} random variable $X_J$ of joint occurrence of $J$ and estimate $\eta(J) = \Pr(X_J = 1)$ from \emph{continuous} data.
The key idea is to use the \emph{normalized ranking} for each feature $j \in V$, which is defined as $\rank(v_i^j) = (k - 1) / (N - 1)$ if $v_i^j$ is the $k$th smallest value among $N$ values $v_1^j$, $v_2^j$, $\dots$, $v_N^j$.
Then we convert a dataset $D$ with respect to features in a combination $J \in 2^V$ into a single $N$-dimensional vector $\vec{x}_{J} = (x_{J1}, x_{J2}, \dots, x_{JN}) \in [0, 1]^N$ by
\begin{align}
 \label{eq:rankvector}
 x_{Ji} = \prod\nolimits_{j \in J} \rank(v_i^j)
\end{align}
for each data point $i \in \{1, 2, \dots, N\}$.
\citet{Tatti13} showed that the empirical estimate of the probability $\Pr(X_J = 1)$ of the joint occurrence of a combination $J$ is obtained by the \emph{copula support} defined as
\begin{align}
 \label{eq:copula}
 \eta(J) = \frac{1}{N} \sum_{i = 1}^N x_{Ji} = \frac{1}{N} \sum_{i = 1}^N \prod_{j \in J} \rank(v_i^j).
\end{align}
Intuitively, joint occurrence of $J$ means that rankings among features in $J$ match with each other.
The definition in Equation~\eqref{eq:copula} is analogue to the support for binary data, where the only difference here is $\rank(v_i^j)$ instead of binary value $v_i^j$.

The copula support always satisfies $\eta(J) \le 0.5$ by definition and has the \emph{monotonicity} with respect to the inclusion relationship of combinations:
$\eta(J) \ge \eta(J \cup \{j\})$
for all $J \in 2^V$ and $j \in V \setminus J$.
Note that, although \citet{Tatti13} considered a statistical test for the copula support, it cannot be used in our setting due to the following two reasons:
(1) his setting is unsupervised while ours is supervised; (2) multiple testing correction was not considered for the test and the FWER was not controlled.

\subsection{Statistical Testing}\label{subsec:test}
Now we can formulate our problem of finding feature interactions, or itemset mining on continuous features, that are associated with class labels as follows:
Let $Y$ be an output binary variable of which class labels are realizations.
The task is to find \emph{all} feature combinations $J \in 2^V$ such that the null hypothesis $X_J \ip Y$, that is, $X_J$ and $Y$ are statistically independent, is rejected by a statistical association test while rigorously controlling the FWER, the probability of detecting one or more false positive associations, under a predetermined significance level $\alpha$.

As a statistical test, we propose to use a \emph{likelihood ratio test}~\cite{Fisher22}, which is a generalized $\chi^2$-test and often called G-test~\cite{Woolf57}.
Although Fisher's exact test has been used as the standard statistical test in recent studies~\cite{Llinares2015KDD,SugiyamaSDM15,Terada13}, it can be applied to only discrete test statistics and cannot be used in our setting.

\begin{table}[t]
 \centering
 \caption{Contingency tables.}
 \label{table:contingency}
 \begin{subtable}[b]{.49\linewidth}
  \centering
  \caption{Expected distribution.}\label{table:contingency:expected}
  \begin{tabular}{cccc}
   \toprule
   & $X_J = 1$ & $X_J = 0$ & Total\\ \midrule
   $Y = 1$ & $\eta(J)r_1$ & $r_1 - \eta(J)r_1$ & $r_1$\\
   $Y = 0$ & $\eta(J)r_0$ & $r_0 - \eta(J)r_0$ & $r_0$\\
   \midrule
   Total   & $\eta(J)$ & $1 - \eta(J)$ & $1$\\
   \bottomrule
  \end{tabular}
 \end{subtable}
 \begin{subtable}[b]{.49\linewidth}
  \centering
  \caption{Observed distribution.}\label{table:contingency:observed}
  \begin{tabular}{cccc}
   \toprule
   & $X_J = 1$ & $X_J = 0$ & Total\\ \midrule
   $Y = 1$ & $\eta_1(J)$ & $r_1 - \eta_1(J)$ & $r_1$\\
   $Y = 0$ & $\eta_0(J)$ & $r_0 - \eta_0(J)$ & $r_0$\\
   \midrule
   Total   & $\eta(J)$ & $1 - \eta(J)$ & $1$\\
   \bottomrule
  \end{tabular}
 \end{subtable}
\end{table}

Suppose that $\Pr(Y = l) = r_l$ for each class $l \in \{0, 1\}$.
From two binary variables $X_J$ and $Y$, we obtain a $2 \times 2$ contingency table, where each cell denotes the joint probability $\Pr(X_J = l, Y = l')$ with $l, l' \in \{0, 1\}$ and can be described as a four-dimensional probability vector $\vec{p}$:
\begin{align*}
 \vec{p} = \Big(\,\Pr(X_J = 1, Y = 1), \Pr(X_J = 1, Y = 0), \Pr(X_J = 0, Y = 1), \Pr(X_J = 0, Y = 0)\,\Big).
\end{align*}
Let $\vec{p}_{\mathrm{E}}$ be the probability vector under the null hypothesis $X_J \ip Y$ and $\vec{p}_{\mathrm{O}}$ be the empirical vector obtained from $N$ observations.
The difference between two distributions $\vec{p}_{\mathrm{O}}$ and $\vec{p}_{\mathrm{E}}$ can be measured by the Kullback--Leibler (KL) divergence
$\KL(\vec{p}_{\mathrm{O}}, \vec{p}_{\mathrm{E}}) = \sum_i p_{\mathrm{O}i} \log (p_{\mathrm{O}i} / p_{\mathrm{E}i})$,
and the independence $X_J \ip Y$ is translated into the condition $\KL(\vec{p}_{\mathrm{O}}, \vec{p}_{\mathrm{E}}) = 0$.
In the G-test, which is a special case of likelihood ratio test, the test statistic is given as $\lambda = 2N\KL(\vec{p}_{\mathrm{O}}, \vec{p}_{\mathrm{E}})$, which follows the $\chi^2$-distribution with the degree of freedom $1$.

In our case, for each combination $J \in 2^V$, the probability vector $\vec{p}_{\mathrm{E}}$ under the null is given as
\begin{align*}
 \vec{p}_{\mathrm{E}} = \bigl(\,\eta(J)r_1,\, \eta(J)r_0,\, r_1 - \eta(J)r_1,\, r_0 - \eta(J)r_0\,\bigr),
\end{align*}
where $\eta(J)$ is the copula support of a feature combination $J$ and $r_l$ is the ratio of the label $l \in \{0, 1\}$ in the dataset.
In contrast, the observed probability vector $\vec{p}_{\mathrm{O}}$ is given as
\begin{align*}
 \vec{p}_{\mathrm{O}} = \bigl(\,\eta_1(J),\, \eta_0(J),\, r_1 - \eta_1(J),\, r_0 - \eta_0(J)\,\bigr),
\end{align*}
where $\eta_1(J) = (1/N)\vec{x}_J \cdot \vec{y}$ and $\eta_0(J) = (1/N)\vec{x}_J \cdot (1 - \vec{y})$ with $\vec{x}_J$ defined in Equation~\eqref{eq:rankvector}.
These two distributions are shown in Table~\ref{table:contingency}.
While it is known that the statistic $\lambda$ does not exactly follow the $\chi^2$-distribution if one of components of $\vec{p}_{\mathrm{E}}$ is too small,
such situation does not usually occur since $\eta(J)$ is not too large as $\eta(J) \le 0.5$ by definition and not too small as such combinations are not testable, which will be shown in Section~\ref{subsec:alg}.

In the following, we write the set of all possible probability vectors by $\vec{\Pc} = \{\,\vec{p} \mid p_i \ge 0, \sum p_i = 1\,\}$, and
its subset given marginals $a$ and $b$ as $\vec{\Pc}(a, b) = \{\,\vec{p} \in \vec{\Pc} \mid p_1 + p_2 = a,\, p_1 + p_3 = b\,\}$.

\subsection{Multiple Testing Correction}\label{subsec:mcp}
Since we have $2^{d}$ hypotheses as each feature combination translated into a hypothesis, we need to perform \emph{multiple testing correction} to control the FWER (family-wise error rate), otherwise we will find massive false positive combinations.
The most popular multiple testing correction method is \emph{Bonferroni correction}~\cite{Bonferroni36}.
In the method, the predetermined significance level $\alpha$ (e.g.~$\alpha = 0.05$) is corrected as $\delta_{\Bon} = \alpha / m$, where $m$ is called a \emph{correction factor} and is the number of hypotheses $m = 2^d$ in our case.
Each hypothesis $J$ is declared as significant only if $\text{\emph{p}-value}(J) < \delta_{\Bon} = \alpha / 2^d$.
Then the resulting FWER $< \alpha$ is guaranteed.
However, it is well known that Bonferroni correction is too conservative.
In particular in our case, the correction factor $2^d$ is too massive due to the combinatorial effect and it is almost impossible to find significant combinations, which will generate massive \emph{false negatives}.

Here we use the \emph{testability} of hypotheses introduced by \citet{Tarone90} and widely used in the literature of significant pattern mining~\cite{Llinares2015KDD,Papaxanthos16,SugiyamaSDM15,Terada13}, which allows us to prune unnecessary feature combinations that can never be significant while controlling the FWER at the same level.
Hence Tarone's testability always offers better results than Bonferroni correction if it is applicable.
The only requirement of Tarone's testability is the existence of the lower bound of the $p$-value given the marginals of the contingency table, which are $\eta(J)$ and $r_0$ (or $r_1$) in our setting.
In the following, we prove that we can analytically obtain the tight upper bound of the KL divergence, which immediately leads to the tight lower bound of the $p$-value.

\begin{theorem}
 \label{theorem:KLmax}
 \textup{({\scshape Tight upper bound of the KL divergence})}
 For $a, b \in [0, 1]$ with $a \le b \le 1/2$ and a probability vector $\vec{p}_{\mathrm{E}} = (ab, a(1 - b), (1 - a)b, (1 - a)(1 - b))$,
 \begin{align}
  \label{eq:KLmax}
  \KL(\vec{p}, \vec{p}_{\mathrm{E}}) < a\log\frac{1}{b} &+ (b - a)\log\frac{b - a}{(1 - a)b} + (1 - b)\log\frac{1}{(1 - a)}
 \end{align}
 for all $\vec{p} \in \vec{\Pc}(a, b)$ and this is tight (see Appendix for its proof).
 \end{theorem}
\textit{Proof}.\hspace*{.5em}
 Let $f(x) = \KL(\vec{p}, \vec{p}_{\mathrm{E}})$ with $\vec{p} = (x, a - x, b - x, (1 - b) - (a - x))$ and $a' = 1 - a$, $b' = 1 - b$. We have
 \begin{align*}
 f(x) =\ &x\log\frac{x}{ab} + (a - x)\log\frac{a - x}{ab'} + (b - x)\log\frac{b - x}{a'b} + (b' - a + x)\log\frac{b' - a + x}{a'b'}.
\end{align*}
 The second derivative of $f(x)$ is given as
\begin{align*}
 \frac{\partial^2 f}{\partial x^2} = \frac{1}{1 - a - b + x} + \frac{1}{b - x} + \frac{1}{a - x} + \frac{1}{x},
\end{align*}
which is always positive from the constraint $0 < x < \min\{a, b\} = a$.
Thus $f(x)$ is maximized when $x$ goes to $0$ or $a$.
The limit of $f(x)$ is obtained as
\begin{align*}
 \lim_{x \to 0} f(x) &= a\log\frac{1}{b'} + b\log\frac{1}{a'} + (b' - a)\log\frac{b' - a}{a'b'},\\
 \lim_{x \to a} f(x) &= a\log\frac{1}{b} + (b - a)\log\frac{b - a}{a'b} + (1 - b)\log\frac{1}{a'}.
\end{align*}
To check which is larger, let $\delta$ be the difference $\lim_{x \to a} f(x) - \lim_{x \to 0} f(x)$. Then it is obtained as
$\delta = - b\log b + b'\log b' + (b - a)\log(b - a) - (b' - a)\log(b' - a)$.
The partial derivative of $\delta$ with respect to $a$ is
$\partial \delta / \partial a = \log(1 - a - b) - \log(b - a)$,
which is always positive as $(1 - a - b) - (b - a) = 1 - 2b \ge 0$ with the condition $b \le 1/2$.
Hence the difference $\delta$ takes the minimum value at $a = 0$ and we obtain
\begin{align*}
 \delta \ge - b\log b + b'\log b' + b'\log b' - b'\log b' = 0.
\end{align*}
Thus $f(x)$ is the tight upper bound when $x \to a$.$\hfill\square$

Here we formally introduce how to prune unnecessary hypotheses by Tarone's testability.
Let $\psi(J)$ be the lower bound of the $p$-value of $J$ obtained from the upper bound of the KL-divergence proved in Theorem~\ref{theorem:KLmax}.
Suppose that $J_1, J_2, \dots, J_{2^d}$ be the sorted sequence of all combinations such that
\begin{align*}
 \psi(J_1) \le \psi(J_2) \le \psi(J_3) \le \dots \le \psi(J_{2^d})
\end{align*}
is satisfied.
Let $m$ be the threshold such that
\begin{align}
 \label{eq:testability}
 m \cdot \psi(J_m) < \alpha \quad\text{and}\quad (m + 1) \cdot \psi(J_{m + 1}) \ge \alpha.
\end{align}
\citet{Tarone90} showed that the FWER is controlled under $\alpha$ with the correction factor $m$.
The set $\Tc = \{J_1, J_2, \dots, J_m\}$ is called \emph{testable} combinations, and each $J \in \Tc$ is significant if $\text{\emph{p}-value}(J) < \delta_{\Tar} = \alpha / m$.
On the other hand, combinations $J_{m + 1}, \dots, J_{2^d}$ are called \emph{untestable} as they can never be significant.
Since $m \ll 2^d$ usually holds, we can expect to obtain higher statistical power in Tarone's method compared to Bonferroni method.
Moreover, we present C-Tarone in the next subsection, which can enumerate such testable combinations without seeing untestable ones, hence we overcome combinatorial explosion of the search space of combinations.

Let us denote by $B(a, b)$ the upper bound provided in Equation~\eqref{eq:KLmax}.
We analyze the behavior of the bound $B(a, b)$ as a function of $a$ with fixed $b$.
This is a typical situation in our setting as $a$ corresponds to the copula support $\eta(J)$, which varied across combinations, while $b$ corresponds to the class ratio $r_1$, which is fixed in each analysis.
Assume that $b \le 1/2$. When $a < b$, we have
$\partial B(a, b) / \partial a = \log (1 - a) / (b - a) > 0$,
hence it is monotonically increases as $a$ increases.
When $b < a < 1/2$, we have
$\partial B(a, b) / \partial a = \log(a - b) / a < 0$,
thereby it monotonically decreases as $a$ increases.
We illustrate the bound $B(a, b)$ with $b = 0.3$ and the corresponding minimum achievable $p$-value with the sample size $N = 100$ in Figure~\ref{figure:klmax_pmin} in Appendix.

\subsection{Algorithm}\label{subsec:alg}

We present the algorithm of C-Tarone that efficiently finds testable combinations $J_1, J_2, \dots$, $J_m$ such that $\psi(J_1) \le \psi(J_2) \le \dots \le \psi(J_m)$, where $m$ satisfies the condition~\eqref{eq:testability}.
We summarize C-Tarone in Algorithm~\ref{alg:enum}, which performs depth-first search to find $J_1, J_2, \dots$, $J_m$ such that $\psi(J_1) \le \psi(J_2) \le \dots \le \psi(J_m)$, where $m$ satisfies the condition~\eqref{eq:testability}.
Suppose that $r_1 \le 1/2 \le r_0$.
Since the lower bound of the $p$-value $\psi(J)$ takes the minimum value when $\eta(J) = r_1$ (see Figure~\ref{figure:klmax_pmin} in Appendix by letting $a = \eta(J)$ and $b = r_1$) and is monotonically decreasing as $\eta(J)$ decreases, for any $\eta(J) \le r_1$, $\Cc \supseteq \Tc$ with $\Cc = \{J \in 2^V \mid \eta(J) \ge \sigma\}$ and $\Tc = \{J \in 2^V \mid \psi(J) \le B(\sigma, r_1)\}$ is always guaranteed, where $\sigma$ is a threshold for copula supports.
Thus if the condition~\eqref{eq:testability} is satisfied for some $m \le |\Tc|$, the $m$ smallest combinations in $\Tc$ are the testable combinations.

\IncMargin{12pt}
\begin{algorithm}[t]
 \SetKwInOut{Input}{input}
 \SetKwInOut{Output}{output}
 \SetFuncSty{textrm}
 \SetCommentSty{textrm}
 \SetKwFunction{SignificantFeatureCombinationSearch}{{\scshape SignificantFeatureCombinationSearch}}
 \SetKwFunction{CTarone}{{\scshape C-Tarone}}
 \SetKwFunction{DFS}{{\scshape DFS}}
 \SetKwFunction{SignificanceTesting}{{\scshape SignificanceTesting}}
 \SetKwProg{myfunc}{}{}{}

 \myfunc{\CTarone{$D$, $\alpha$}}{
  $\sigma \gets 0$;\quad $\Cc \gets \emptyset$;\tcp*[f]{$\sigma$ is a global variable}\\
  \DFS{$\emptyset$, $0$, $\Cc$, $D$, $\alpha$};\\
  $\Tc \gets \{K \in \Cc \mid \psi(K) \le B(\sigma, r_1)\}$;\\
  \tcp*[f]{The set of testable combinations}\\
 \SignificanceTesting{$\Tc$, $\alpha$};
 }
 \myfunc{\DFS{$J$, $j_{\mathrm{prev}}$, $\Cc$, $D$, $\alpha$}}{
 \ForEach{$j \in \{j_{\mathrm{prev}} + 1, \dots, d\}$}{
 $J \gets J \cup \{j\}$;\\
  Compute $\eta(J)$ by Equation~\eqref{eq:copula};\\
 \If{$\eta(J) > \sigma$}{
 Compute $\psi(J)$;\ \  $\Cc \gets \Cc \cup \{J\}$;\\
 $\Tc \gets \{K \in \Cc \mid \psi(K) \le B(\sigma, r_1)\}$;\\
 \While{$|\Tc| B(\sigma, r_1) \ge \alpha$}{
 $J_{\min} \gets \text{argmin}_{K \in \Cc} \eta(K)$;\\
  $\sigma \gets \eta(J_{\min})$;\ \  $\Cc \gets \Cc \setminus \{J_{\min}\}$;\\
  $\Tc \gets \Tc \setminus \{J_{\min}\}$;\\
 }
 \DFS{$J$, $j$, $\Cc$, $D$, $\alpha$};
 }
 $J \gets J \setminus \{j\}$;
 }
 }
 \myfunc{\SignificanceTesting{$\Tc$, $\alpha$}}{
 \ForEach{$J \in \Tc$}{
 \lIf{$p\text{-value}(J) < \alpha \,/\, |\Tc|$}{output $J$}
 }
 }
 \caption{C-Tarone.}
 \label{alg:enum}
\end{algorithm}
\DecMargin{12pt}

Moreover, since $\eta(J)$ has the monotonicity with respect to the inclusion relationship, that is, $\eta(J) \ge \eta(J \cup \{j\})$ for all $j \in V \setminus J$, finding the set $\Cc$ is simply achieved by DFS (depth-first search): Starting from the smallest combination $\emptyset$ with assuming $\eta(\emptyset) = 1$,
if $\eta(J) \ge \sigma$ for a combination $J$, we update $J$ by adding $j \in V \setminus J$ and recursively check $J$. Otherwise if $\eta(J) < \sigma$, we prune all combinations $K \supseteq J$ as $\eta(K) < \sigma$ holds.

Our algorithm dynamically updates the threshold $\sigma$ to enumerate testable combinations while pruning massive unnecessary combinations.
First we set $\sigma = 0$, which means that all combinations are testable.
Whenever we find a new combination $J$, we update $\Cc$ and $\Tc$ and check the condition~\eqref{eq:testability}.
If $|\Tc| B(\sigma, r_1) > \alpha$ holds, $\sigma$ is too low and $\Tc$ will become too large with the current $\sigma$, hence we update $\sigma$ to $\min_{J \in \Cc} \eta(J)$ and remove the combination $\argmin_{J \in \Cc} \eta(J)$.
Finally, when the algorithm stops, it is clear that the set $\Tc$ coincides with the set of testable combinations, hence each combination $J \in \Tc$ is significant if the $p$-value of $J$ is smaller than $\delta_{\Tar} = \alpha / |\Tc|$.

Since Algorithm~\ref{alg:enum} always outputs the set of testable combinations $\Tc$, which directly follows from $\Cc \supseteq \Tc$ and the monotonicity of $\eta(J)$, the {\scshape SignificanceTesting} function finds all significant combinations with the FWER $\le \alpha$.
Thus Algorithm~\ref{alg:enum} always finds all significant feature combinations with controlling the $FWER$ under $\alpha$.
Moreover, C-Tarone is independent of the feature ordering and the above completeness with respect to the set of significant combinations is always satisfied.
The time complexity of C-Tarone is $O(|\Tc|)$.

\begin{figure}[t]
 \centering
 \includegraphics[width=.6\linewidth]{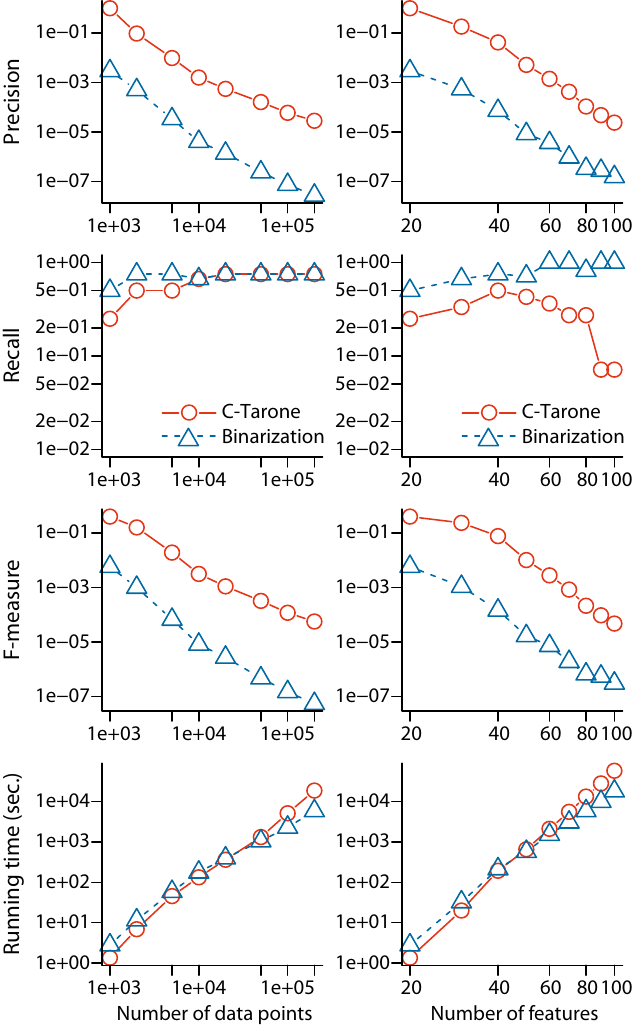}
 \caption{Results on synthetic data with the minor class ratio $r_1 = 0.5$.
 Regarding the scale of precision and F-measure, see comment at the last paragraph just before Section~\ref{subsec:synth}.
 The number of features is $d = 20$ in the left column and the sample size is $N = \text{1,000}$ in the right column.
 Both x- and y-axes are in logarithmic scale. C-Tarone is shown in red circles, the binarization approach in blue triangles.}
 \label{figure:synth_all}
\end{figure}

\begin{figure*}[t]
 \centering
 \includegraphics[width=\linewidth]{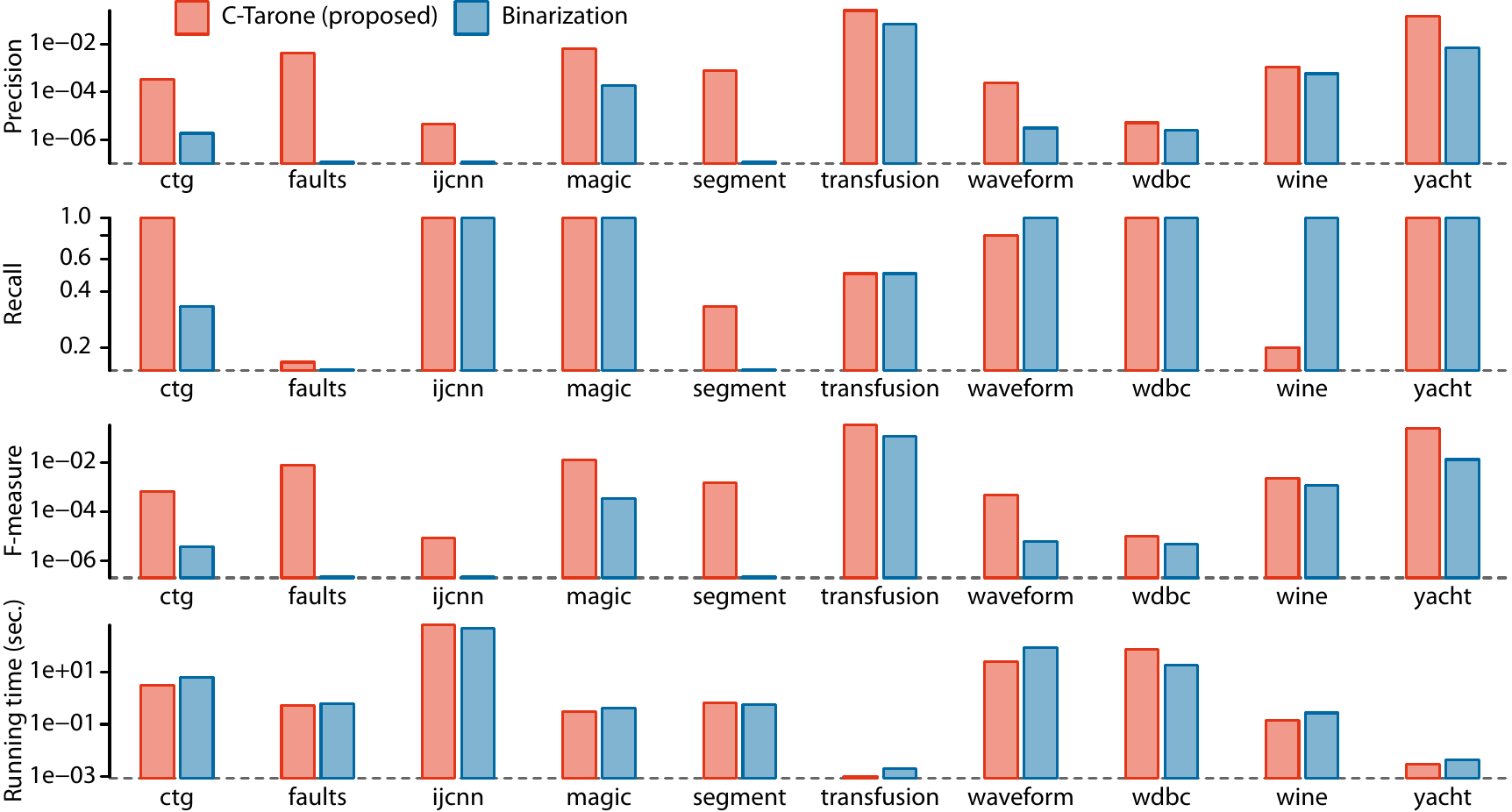}
 \caption{Results on real data. Regarding the scale of precision and F-measure, see the comment at the last paragraph just before Section~\ref{subsec:synth}. The y-axis is in logarithmic scale. C-Tarone is shown in red and the binarization approach is shown in blue. Higher (taller) is better in precision, recall, and F-measure, while lower is better in running time.}
 \label{figure:real_results}
\end{figure*}

\section{Experiments}\label{sec:exp}
We examine the effectiveness and the efficiency of C-Tarone using synthetic and real-world datasets.
We used Amazon Linux AMI release 2017.09 and ran all experiments on a single core of 2.3 GHz Intel Xeon CPU E7-8880 v3 and 2.0 TB of memory.
All methods were implemented in \texttt{C}/\texttt{C++} and compiled with {\ttfamily gcc} 4.8.5.
The FWER level $\alpha = 0.05$ throughout experiments.

There exists no existing method that can enumerate significant feature combinations from continuous data with multiple testing correction.
Thus we compare C-Tarone to \emph{significant itemset mining} method with prior \emph{binarization} of a given dataset since significant itemset mining offers to enumerate all significant feature combinations while controlling the FWER from binary data.
We employ median-based binarization as a preprocessing.
For each feature $j \in V$, we pick up the median of $\vec{v}^j = (v_1^j, v_2^j, \dots, v_N^j)$, denoted by $\med(j)$, and
binarize each value $v_i^j$ as a pair $(v_i^{\le \med(j)}, v_i^{> \med(j)})$, where $v_i^{\le \med(j)} = \mathtt{1}$ if $v_i^{j} \le \med(j)$ and $\mathtt{0}$ otherwise, and $v_i^{> \med(j)} = \mathtt{1}$ if $v_i^{j} > \med(j)$ and $\mathtt{0}$ otherwise.
Thus a given dataset is converted to the binarized dataset with $2d$ features.
We use the state-of-the-art significant itemset mining algorithm LAMP ver.2~\cite{Minato14} and employ implementation provided by~\cite{Llinares18CASMAP}, which incorporates the fastest frequent pattern mining algorithm LCM~\cite{Uno04} to enumerate testable feature combinations from binary data.
Note that, in terms of runtime, comparison with the brute-force approach of testing all the $2^d$ combinations with the Bonferroni correction is not valid as the resulting FWER is different between the brute-force and the proposed method.

We also tried other binarization approaches, \emph{interordinal scaling} used in numerical pattern mining~\cite{Kaytoue11} and \emph{interval binarization} in subgroup discovery~\cite{Grosskreutz09} as a preprocessing of significant itemset mining.
However, both preprocessing generate too many binarized dense features, resulting in the lack of scalability in the itemset mining step in the enumeration of testable combinations.
Hence we do not employ them as comparison partners.
Details of these binarization techniques are summarized in Appendix.

To evaluate the efficiency of methods, we measure the running time needed for enumerating all significant combinations.
In the binarization method, we exclude the time used for binarization as this preprocessing step is efficient enough and negligible compared to the pattern mining step.

To examine the quality of detected combinations, we compute precision, recall, and the F-measure by comparing such combinations with those obtained by the standard decision tree method CART~\cite{Breiman84}, which obtains multiplicative combinations of features in the form of binary trees.
We used the \texttt{rpart} function in R with its default parameter setting, where the Gini index is used for splitting and the minimum number of data points that must exist in a node is $20$.
We apply the decision tree to each dataset and retrieve all the paths from the root to leaf nodes of the learned tree.
In each path, we use the collection of features used in the path as a positive feature combination of the ground truth, that is, a feature combination found by C-Tarone (or binarization method) is deemed to be true positive if and only if it constitutes one of full paths from the root to a leaf of the learned decision tree.
Note that the FWER is always controlled under $\alpha$ in both of C-Tarone and binarization method, and our aim is to empirically examine the quality of the feature combinations compared to those selected by a standard decision tree.
We used not forests but a single tree as the ground truth depends on the number of trees if we use forests, resulting in arbitrary results.

Please note that, due to the fact that there are up to $2^d$ feature combinations and we only count an exact match between a retrieved pattern and a true pattern as a hit, precision and F-measure will be close to 0.
Still, we compute them to allow for a comparison of the relative performance of the different approaches.
Evaluation criteria that take overlaps between retrieved patterns and true patterns (partial matches) into account would lead to higher levels of precision, but they are a topic of ongoing research and not a focus of this work.

\textbf{Results on Synthetic Data.}\label{subsec:synth}
First we evaluate C-Tarone on synthetic data with varying the sample size $N$ from $1,000$ to $200,000$, the number $d$ of features from $20$ to $100$, and setting the class ratio to $r_1 = 0.5$ or $r_1 = 0.2$, i.e., the number $N_1$ of samples in the minor class is $N/2$ or $N/5$.
In each dataset, we generate 20\% of features that are associated with the class labels.
More precisely, first we generate the entire dataset from the uniform distribution from $0$ to $1$ and assign the class label $1$ to the first $N_1$ data point.
Then, for the $N_1$ data points in the class $1$, we pick up one of the 20\% of associated features and copy it to every associated feature with adding Gaussian noise with $(\mu, \sigma^2) = (0, 0.1)$. Hence there is no correlation in the class $0$ across all features and there are positive correlations among such 20\% of features in the class $1$.
The other 80\% are uninformative features.

Results are plotted in Figure~\ref{figure:synth_all} for $r_1 = 0.5$ (classes are balanced).
See Figure~\ref{figure:synth_02} in Appendix for $r_1 = 0.2$ (classes are imbalanced).
In the figure, we plot results with varying $N$ while fixing $d = 20$ on the left column and those with varying $d$ while fixing $N = 1,000$ on the right column.

In comparison with the median-based binarization method (plotted in blue), C-Tarone (plotted in red) has a clear advantage regarding the precision, which is several orders of magnitude higher than the binarization method in every case.
This is why binarization method cannot distinguish correlated and uncorrelated combinations as we discussed in Introduction and illustrated in Figure~\ref{figure:example}, resulting in including uncorrelated features into significant combinations.
Although recall is competitive across various $N$ and $d$, in all cases, the F-measure of C-Tarone are higher than those of the median-based binarization method.
In both methods, precision drops when the sample size  $N$ becomes large: As we gain more and more statistical power for larger $N$, many feature combinations, even those with very small dependence to the class labels and not used in the decision tree, tend to reach statistical significance.

Although the algorithm LCM (itemset mining) used in the binarization approach is highly optimized with respect to the efficiency, C-Tarone is competitive with it on all datasets, as can be seen in Figure~\ref{figure:synth_all} (bottom row).
To summarize, we observe that C-Tarone improves over the competing binarization approach in terms of the F-measure in detecting higher quality feature combinations in classification.

\textbf{Results on Real Data.}
We also evaluate C-Tarone on real-world datasets shown in Table~\ref{table:real_stat} in Appendix, which are benchmark datasets for binary classification from the UCI repository~\cite{Lichman13}.
To clarify the exponentially large search space, we also show the number $2^d$ of candidate combinations for $d$ features in the table.
All datasets are balanced to maximize the statistical power for comparing detected significant combinations, i.e.\ $r_1 = 0.5$.
If they are not balanced in the original dataset, we randomly subsample data points from the larger class.

We summarize results in Figure~\ref{figure:real_results}.
Again, C-Tarone shows higher precision in all datasets than the binarization method and better or competitive recall, resulting in higher F-measure scores in all datasets.
In addition, running time is competitive with the binarization method, which means that C-Tarone can successfully prune the massive candidate space for significant feature combinations.
These results demonstrate the effectiveness of C-Tarone.

\section{Conclusion}\label{sec:conclusion}
In this paper, we have proposed a solution to the open problem of finding all multiplicative feature interactions between \emph{continuous} features that are significantly associated with an output variable after rigorously controlling for multiple testing. While interaction detection with multiple testing has been studied before, our approach, called \emph{C-Tarone}, is the first to overcome the problem of detecting all higher-order interactions from the enormous search space $2^d$ for $d$ features.

Our work opens the door to many applications of searching significant feature combinations, in which the data is not adequately described by binary features, including large fields such as data analysis for high-throughput technologies in biology and medicine.
Our work here addresses the problem of finding continuous \emph{features}.
Finding significant combinations associated with \emph{continuous output variables} is an equally challenging and practically relevant problem, that we will tackle in future work.

%\bibliography{main}

\appendix

\section{Related Work}\label{sec:related}
Since the only line of work that tries to find multiplicative feature combinations while controlling the FWER is \emph{significant pattern mining}, we provide an overview of this field in the following.

Significant pattern mining introduces statistical significance into the task of contrast (or {\it discriminative}) pattern mining~\cite{Dong13}, where the objective is to find discriminative patterns with respect to class partitioning of a dataset.
After early work on multiple testing correction in association rule mining~\cite{Hamalainen12,Webb07}, \citet{Terada13}~were the first to achieve control of the FWER in itemset mining, by successfully combining a pattern mining algorithm and Tarone's trick~\cite{Tarone90}.
The enumeration algorithm has been improved in LAMP ver.~2~\cite{Minato14} and significant subgraph mining~\cite{SugiyamaSDM15}.
To date, significant pattern mining has been extended to various types of tests and data, including a Westfall-Young permutation test to treat independencies among patterns~\cite{Llinares2015KDD,Terada13BIBM}, logistic regression~\cite{Terada16PAKDD} or a Cochran--Mantel--Haenszel (CMH) test~\cite{Llinares17,Papaxanthos16} for categorical covariates, hypothesis streams~\cite{Webb16}, and top-$K$ significant patterns~\cite{Pellegrina18}.
However, none of the above studies succeeded to directly perform significant pattern mining on continuous variables without prior binarization.

Although the field of \emph{subgroup discovery}~\cite{Atzmueller15,Novak09,Herrera11} also considers measures of statistical dependence for finding multiplicative feature combinations, e.g.~\cite{Grosskreutz09, Mampaey15,vanLeeuwen16}, none of these methods accounts for multiple testing by controlling the FWER.

\begin{figure}[t]
 \centering
 \includegraphics[width=.8\linewidth]{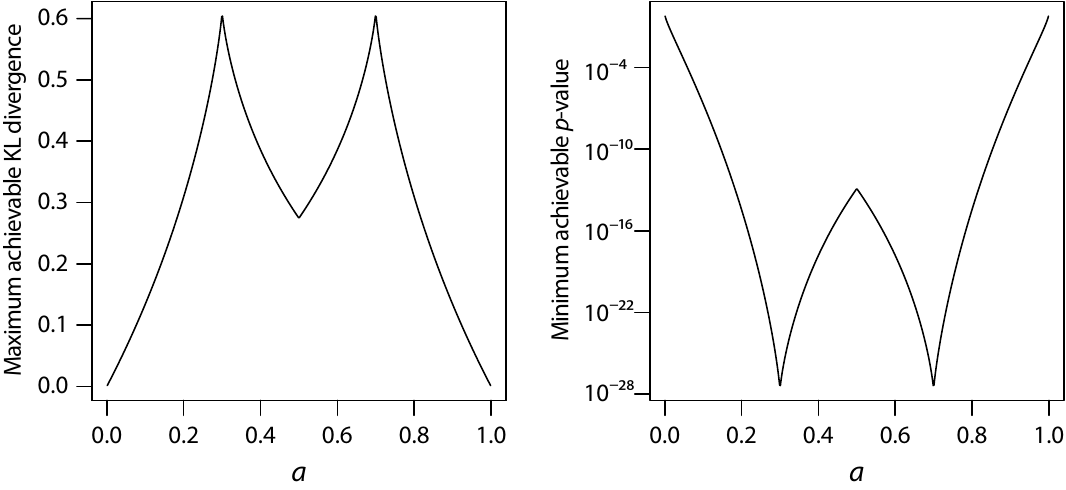}
 \caption{The upper bound $B(a, b)$ of the KL divergence (left) and the corresponding \textit{p}-value (right) with $N = 100$ with respect to changes in $a$ when $b = 0.3$.}
 \label{figure:klmax_pmin}
\end{figure}

\section{Additional Binalization Methods}
In interordinal scaling, each binarized feature is in the form of ``$\le a$'' or ``$\ge a$'', where endpoints $a$ are from a dataset, that is, $a \in \{v_1^j, v_2^j, \dots, v_N^j\}$ for a feature $j \in \{1, 2, \dots, d\}$.
Thus, for an $d$-dimensional real-valued vector $\vec{v}_i \in D$, each element $v_i^j$ is expanded as the $2N$-dimensional binary vector such that
\begin{align*}
 \Bigl(\,v_i^{\le v_1^j}, v_i^{\le v_2^j}, \dots, v_i^{\le v_N^j}, v_i^{\ge v_1^j}, v_i^{\ge v_2^j}, \dots, v_i^{\ge v_N^j}\,\Bigr),
\end{align*}
where each value for the binarized feature $v_i^{\le v_k^j} = \mathtt{1}$ if $v_i^j \le v_k^j$ and \texttt{0} otherwise.
As a result, the dataset $D$ is converted into the binary dataset with $2dN$ features.
In interval binarization, each binarized feature is in the form of ``$(a, b]$'', where endpoints $a$, $b$ are from data, and
each element $v_i^j$ of an $d$-dimensional vector $\vec{v}_i$ is expanded as the $(N(N - 1) / 2)$-dimensional binary vector such that
\begin{align*}
 \Bigl(\,v_i^{(v_1^j,v_2^j]}, v_i^{(v_1^j,v_3^j]}, \dots, v_i^{(v_1^j,v_N^j]}, v_i^{(v_2^j,v_3^j]}, \dots, v_i^{(v_{N - 1}^j,v_N^j]}\,\Bigr).
\end{align*}
Thus a dataset $D$ is converted into the binary dataset with $dN(N - 1)/2$ features.
Both interordinal scaling and interval binarization could finish their computation for a tiny dataset with $(N, d) = (50, 5)$ in approximately 24 hours, but did not finish after 48 hours for $(N, d) = (100, 10)$, and they exceeded the memory limit (2.0~TB) for larger datasets.

\begin{figure}[t]
 \centering
 \includegraphics[width=.6\linewidth]{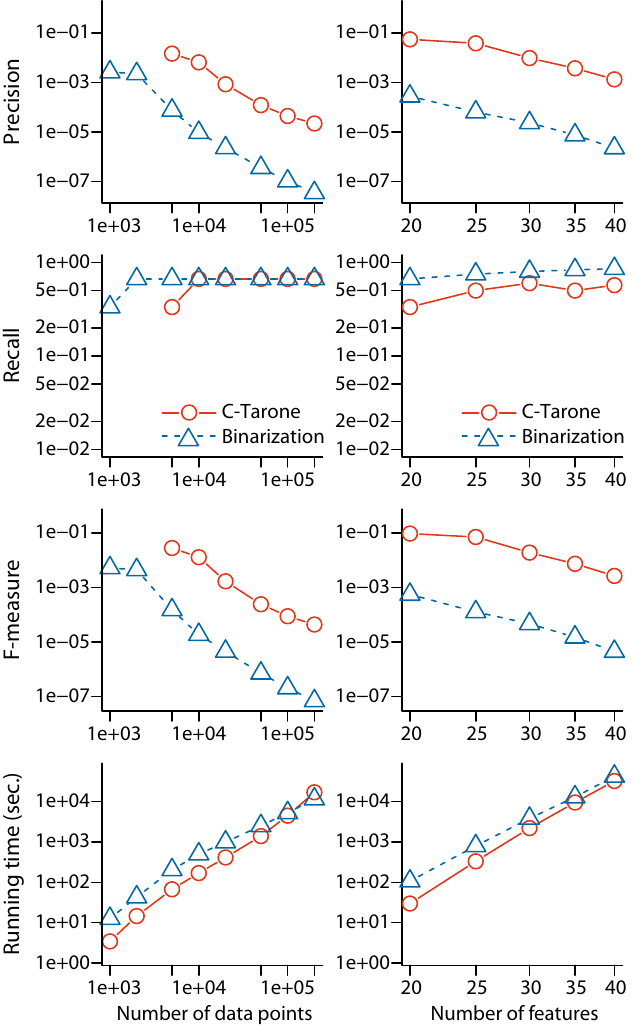}
 \caption{Results on synthetic data with the minor class ratio $r_1 = 0.2$.
 The number of features is $d = 20$ in the left column and the sample size is $N = \text{3,000}$ in the right column.
 Both x- and y-axes are in logarithmic scale. C-Tarone is shown in red circles, the binarization approach in blue triangles.
 Missing points in $(\textbf{b})$ mean that no significant combination is detected.}
 \label{figure:synth_02}
\end{figure}

\begin{table}
\centering
\caption{Statistics of real data.}
\label{table:real_stat}
\begin{tabular}[t]{lrrr}
 \toprule
 Data & $N$ & $d$ & \# candidate combinations\\
 &&& (search space)\\
 \midrule
 ctg & 942 & 22 & 4,194,304\\
 faults & 316 & 27 & 134,217,728\\
 ijcnn & 9,706 & 22 & 4,194,304\\
 magic & 13,376 & 10 & 1,024\\
 segment & 660 & 19 & 524,288\\
 transfusion & 356 & 4 & 16\\
 waveform & 3,314 & 21 & 2,097,152\\
 wdbc & 424 & 30 & 1,073,741,824\\
 wine & 3,198 & 11 & 2,048\\
 yacht & 308 & 6 & 64\\
 \bottomrule
\end{tabular}
\end{table}

\end{document}